\documentclass{article}
\usepackage[a4paper, total={6in, 8in}]{geometry}

\usepackage{tikz}
\usepackage{amsmath}
\usepackage{amsthm}
\usepackage{hyperref, cleveref}
\hypersetup{
    colorlinks=true,
    linkcolor=blue,
    filecolor=blue,      
    urlcolor=blue,
    citecolor=blue,
    pdfpagemode=FullScreen,
    }
\usepackage{amsfonts}
\usepackage{xcolor}
\usepackage[shortlabels]{enumitem}

\usepackage[utf8]{inputenc}
\usepackage{indentfirst}
\usepackage{amsbsy,amssymb,amsthm,amsmath, amsfonts,fixltx2e,bbding} 
\usepackage{hyperref, cleveref}
\usepackage{graphicx}
\usepackage{ragged2e}
\makeatletter
\newcommand*{\rom}[1]{\expandafter\@slowromancap\romannumeral #1@}
\makeatother

\usepackage[ruled, lined, linesnumbered, longend]{algorithm2e}

\title{Simple online learning with consistent oracle}

\author{Alexander Kozachinskiy$^{12}$, Tomasz Steifer$^{134}$}

\date{%
    $^1$Instituto Milenio Fundamentos de los Datos, Chile\\%
    $^2$Centro Nacional de Inteligencia Artificial, Chile\\%
    $^3$Instituto de Ingeniería Matemática y Computacional, Universidad Católica de Chile\\
    $^4$Institute of Fundamental Technological Research, Polish Academy of Sciences\\%
}
\newtheorem{definition}{Definition}
\newtheorem{theorem}{Theorem}
\newtheorem{lemma}{Lemma}

\newtheorem{prop}{Proposition}
\newtheorem{open}{Open Problem}
\newtheorem{remark}{Remark}

\newcommand{\ldim}{\mathsf{Ldim}}

\newcommand{\cons}{{\textnormal{\texttt{ConsOracle}}}}

\newcommand{\createadv}[1]{\textnormal{\texttt{CreateAdv}($#1$)}}
\newcommand{\voteandupdate}[1]{\textnormal{\texttt{VoteAndUpdate}($#1$)}}

\newcommand{\proc}[1]{\textnormal{\texttt{Procedure}($#1$)}}

\newcommand{\predict}{\textnormal{\texttt{Learner}}}

\begin{document}
\maketitle
\begin{abstract}
We consider online learning in the model where a learning algorithm can access the class only via the \emph{consistent oracle}---an oracle, that, at any moment, can give a function from the class that agrees with all examples seen so far. This model was recently considered by Assos et al.~(COLT'23). It is motivated by the fact that standard methods of online learning rely on computing the Littlestone dimension of subclasses, a computationally intractable problem.

Assos et al.~gave an online learning algorithm in this model that makes at most $C^d$ mistakes on classes of Littlestone dimension $d$, for some absolute unspecified constant $C > 0$. We give a novel algorithm that makes at most $O(256^d)$ mistakes. Our proof is significantly simpler and uses only very basic properties of the Littlestone dimension. We also show that there exists no algorithm in this model that makes less than $3^d$ mistakes.

 
\end{abstract}

\section{Introduction}

\paragraph*{Overview of online learning.}

Online learning is one of the most influential theoretical models in machine learning. In this model, introduced by Littlestone~\cite{littlestone1988learning} and, in different terms, by Angluin~\cite{angluin1988queries}, a learner tries to predict values of some objective function $f\colon X \to\{0, 1\}$, not known to the learner, belonging to some hypothesis class $H\in\{0, 1\}^X$, known to the learner. The learner receives inputs to $f$ for prediction in an order, chosen by an adversary. Predictions have to be made in an ``online fashion''. More specifically, in each ``round'' of prediction, the learner receives  $x\in X$, then the learner outputs its prediction for the value of $f$ on $x$, and after that $f(x)$ is revealed to the learner. This repeats for infinitely many rounds.

Eventually, the learner wants to be able to predict the values of $f$ correctly every time. Before that, the learner might experience several rounds with mistakes, when their predictions did not coincide with the real value of $f$. The question is how to use rounds with mistakes most efficiently and ``stabilize'' on $f$ faster, after as few mistakes as possible.

The most basic measure of the efficiency of a learning algorithm in this model is its worst-case number of mistakes, with the maximum taken over all objective functions $f$ from the hypothesis class $H$, and all possible ordering of the inputs.
Littlestone~\cite{littlestone1988learning} characterized hypothesis classes $H$ for which there exists a learning algorithm with finite worst-case mistake bound. Such classes are now known as Littlestone classes. More specifically, he has shown that the minimal achievable worst-case mistake bound of an online learning algorithm for a hypothesis class $H$ is equal to a combinatorial measure, now called the \emph{Littlestone dimension} of $H$.

\medskip

Over the years, online learning has accumulated more and more interest.  Since the 90s, it has been studied in connection to PAC learning, proposed by Valiant~\cite{valiant1984theory}, another classical theoretical model in machine learning. In the PAC model, given an objective function $f$ from a hypothesis class $H$ (as in online learning, $H$ is known to the learner while $f$ is not), the learner can sample data points of the form $(x, f(x))$ from some unknown distribution $\mu$ over the set of inputs $x\in X$, chosen by an adversary. The goal of the learner is to provide, with high probability, a function $g$ which approximates $f$ up to a bounded error. Furthermore, this has to be done after seeing a fixed number of data points (where this number is independent of the choice of underlying distribution $\mu$). 


Hypothesis classes $H$ that admit a learning algorithm in the PAC model have been characterized as classes having finite \emph{Vapnik-Chervonenkis dimension}. Vapnik-Chervonenkis dimension is a lower bound on the Littlestone dimension, meaning that every Littlestone class is PAC learnable. Moreover, there is a polynomial-time reduction,  turning any online learning algorithm for a class $H$ into a  PAC-learning algorithm for $H$~\cite{littlestone1989line}.

Recently, Alon et al.~\cite{alon2022private} have shown that Littlestone classes are exactly classes that admit \emph{approximately differently private} PAC learning algorithm, that is, a (randomized) PAC learning algorithm that only mildly changes its output distribution under changing one point in a sample. Differential privacy can be seen as a notion of stability of the learner and indeed, the result of Alon et al.~was followed by a series of related papers, which characterized Littlestone classes via different notions of stability: such as replicability~\cite{impagliazzo2022reproducibility,bun2023stability}, statistical indistinguishability~\cite{kalavasis2023statistical} or finite information complexity~\cite{pradeep2022finite}.

\medskip
\paragraph*{Oracles for online learning}
These results all use some kind of reduction to the online learning setting, motivating the need for \emph{time-efficient online learning algorithms}. The algorithm of Littlestone~\cite{littlestone1988learning}, now usually called the Standard Optimal Algorithm (SOA), achieves the optimal mistake bound but it is not necessarily time-efficient. The problem is that SOA relies on the oracle that computes the Littlestone dimension of subclasses of $H$. Such oracle in some settings is known to be computationally intractable~\cite{frances1998optimal,manurangsi2017inapproximability,manurangsi2023improved,hasrati2023computable}. For instance, if the set $X$ and the hypotheses class $H$ are finite and given as a truth table of size $n$, then, under standard cryptographic assumptions, there is no algorithm computing Littlestone dimension in $n^{o(\log n)}$ time~\cite{frances1998optimal}. In the case when $X$ and $H$ are infinite SOA might be uncomputable, even if $H$ is given as a decidable set of programs~\cite{hasrati2023computable}.

Assos et al.~\cite{assos2023online} suggested developing online learning algorithms that potentially attain non-optimal mistake bounds but use more feasible oracles. One such oracle, introduced by them, is the \emph{consistent oracle}. A consistent oracle receives on input a \emph{sample}, that is, a sequence of $m$ ``input-output'' pairs:
\[(x_1, y_1), \ldots, (x_m, y_m)\in X\times \{0, 1\}.\]
If this sample is \emph{realizable}, that is, there is a function $f\in H$ with $f(x_1) = y_1, \ldots, f(x_m) = y_m$, the consistent oracle finds some such $f$ in $H$. If the sample is not realizable, the consistent oracle is undefined on it (we assume that running the consistent oracle on the non-realizable sample leads to an infinite loop).

Consistent oracle can be seen as a weak form of the more standard Empirical Risk Minimizer (ERM) oracle. ERM, given a sample $S$, finds a function from the class $H$ which achieves the minimal empirical error possible on $S$. Consistent oracle gives a function whenever it is possible to give one that achieves zero empirical error, but otherwise can be undefined. In general, the existence of an efficient consistent oracle is a weaker requirement than the existence of an efficient ERM.


We are aware of two frameworks where consistent
oracle is efficiently computable while SOA is not. The first one is when the domain is finite and the hypotheses $H$ are given as a truth table. Then consistent oracle can be computed in polynomial time (in the size of the table) while SOA, under standard complexity assumptions, requires quasi-polynomial time~\cite{frances1998optimal,manurangsi2017inapproximability,manurangsi2023improved}. The other framework is concerned with \emph{recursively enumerably representable} (RER) hypothesis classes, studied by Hasrati and Ben-David~\cite{hasrati2023computable}.
 These are classes that can be given as recursively enumerable sets of programs that compute functions from the class. As Hasrati and Ben-David observed, there are RER classes of finite Littlestone dimension for which there exists no computable SOA. On the other hand, a RER class always has a computable consistent oracle. Indeed, we can start enumerating programs for functions of the class until we find one that agrees with all labels of Adversary so far. Since Adversary is constrained to be consistent with some function from the class, this procedure always halts.

 \paragraph*{Our results}
Following Assos et al., we are concerned with online learning algorithms that know nothing about the hypothesis class $H$ except that they have access to a consistent oracle for $H$. This model can be formally defined as a game, where a learning algorithm plays against the adversary who gives inputs for prediction and reveals their labels, and also answers queries to the consistency oracle. In the game, there is no fixed hypothesis class $H$. Rather, in the game that we call \emph{the consistent oracle model for classes of Littlestone dimension $d$},  the adversary has the following restriction -- at any moment, the Littlestone dimension of the set of functions the adversary has used as answers to queries to the consistency oracle cannot exceed $d$.
 
 Assos et al.~gave an online learning algorithm in the consistent oracle model for classes of Littlestone dimension $d$ that makes at most $C^d$ mistakes, where $C > 0$ is some absolute constant (unspecified in their paper). Our main contribution is a new, simpler algorithm that achieves the mistake bound of at most $O(256^d)$. Our result follows from an elementary, self-contained proof that uses only basic facts about the Littlestone dimension. This contrasts with a rather complicated combinatorial proof given by Assos et al., which requires advanced results on threshold dimension and various notions of fat-shattering dimensions. Arguably, our algorithm is itself simpler. For any prediction, it just takes a simple majority vote over some functions, obtained from the consistent oracle on previous steps. At the same time, the algorithm of Assos et al.~uses a weighted majority vote, with exponential weights that are updated in each round. 
 
Our analysis is carried out for improper learning only, i.e., the case when we allow the learner to predict using functions from outside the class. Deterministic proper learning is not always possible for Littlestone classes~\cite{angluin1988queries,chase2020bounds}. However, the algorithm of Assos et al.~admits randomization that turns it into a randomized proper learner with probabilistic bounds on the expected number of mistakes. We leave it as an open problem whether our algorithm could be into randomized proper learner.

Although we only have exponential mistake bounds for both algorithms, there is a chance that in practice the number of mistakes will be polynomial. In any case, as Assos et al.~point out, there might be more sense in trying to run an algorithm with an exponential mistake bound but with a feasible oracle than trying to run SOA where already a single call to the oracle might be unfeasible. Nevertheless, to better understand the model,
 we complement our algorithm with a formal exponential lower bound. We show that there is no online learning algorithm in the consistent oracle model for classes of Littlestone dimension $d$ that makes less than $3^d$ mistakes. We leave open a question of what is the optimal exponent, achievable in the consistent oracle model for Littlestone dimension $d$.
 \begin{open}
     What is the minimal $C > 0$ such that for all $d$ there exists an online learning algorithm in the model with consistent oracle, making at most $O(C^d)$ mistakes for classes of Littlestone dimension $d$?
 \end{open}
 From our results, it follows that $3\le C\le 256$.

 \paragraph*{Applications.} Both our algorithm and the algorithm of Assos et al.~are time-efficient -- to produce a prediction, they need time which is linear in the number of mistakes so far. Thus, both algorithms imply efficient online learners for the two settings mentioned above. Namely, for the setting when $H$ is given as a truth table, we get a polynomial-time online learning algorithm with at most $O(256^d)$ mistakes. In turn, for every RER class of finite Littlestone dimension, we get the existence of a computable online learner. 

 The above results are stated for \textit{realizable} setting, that is when the sample is guaranteed to be consistent with some function from the class. Both our algorithm and that of Assos et al.~can fail to produce an output on nonrealizable samples. Under certain additional assumptions, both algorithms can be turned into randomized agnostic learners---i.e. learners which are defined on nonrealizable samples and which have sublinear expected regret---using the standard prediction with experts approach of~\cite{ben2009agnostic}. This is possible only if the algorithms can be made total---for instance, in case when we have additional access to an oracle deciding for a given sample if it is realizable or not. For example, this is true for the finite setting when $H$ is given as a truth table.

 \paragraph*{Organization} We give Preliminaries in \Cref{sec_prel}. In \Cref{sec_model}, we give a formal definition of the consistent oracle model and show a simple $2^d$ lower bound on the number of mistakes for Littlestone dimension $d$ in this model. In \Cref{sec_alg}, we give an algorithm, achieving $O(256^d)$ mistake bound. Finally, in \Cref{sec_lower}, we give the $3^d$ lower bound.

\section{Preliminaries}
\label{sec_prel}

Here we give necessary definitions, regarding hypothesis classes and Littlestone dimension. Fix an infinite set $X$ called \emph{the domain}. Non-empty sets $H\subseteq \{0, 1\}^X$ will be called \emph{hypothesis classes}.
A \emph{labeled tree} is a complete rooted binary tree of depth $d$, for some $d\in\mathbb{N}$, in which every non-leaf node is labeled by an element of $X$. We also assume that, for every non-leaf node $v$, if we consider two edges that go from $v$ to its children, one of them is labeled by $0$ and the other one is labeled by $1$. Correspondingly, one child of $v$ will be called the $0$-child of $v$ and the other child will be called the $1$-child of $v$. Every leaf $l$ in such a tree can be understood as an assignment of elements of $X$, appearing on the path to $l$, to 0s and 1s. The idea is that when we descend from a node, labeled by $x\in X$, to one of the children of this node, namely, to its $y$-child for some $y\in\{0, 1\}$, this can be understood as if we assign $y$ to $x$. More specifically, any leaf $l$ can be assigned a sequence of pairs $(x_1, y_1), \ldots, (x_{d}, y_{d})\in X\times\{0, 1\}$, where $x_i\in X$ is the label of the depth-$(i-1)$ ancestor of $l$ (the depth-0 ancestor of $l$ is the root), and $y_i\in\{0, 1\}$ is the label of the edge, leading from the depth-$(i-1)$ ancestor of $l$ to its depth-$i$ ancestor. Next, we say that a leaf $l$ is \emph{consistent} with a function $f\colon X\to\{0, 1\}$ if so is the corresponding partial assignment, that is, if 
\[f(x_1) = y_1, \ldots, f(x_{d}) = y_{d}.\]
Now, we say that a hypothesis class $H$ shatters a labeled tree $T$ if every leaf $l$ is consistent with some function in $H$.  

The Littlestone dimension of a hypothesis class $H$, denoted by $\ldim(H)$,  is the maximal $d\ge 1$ such that $H$ shatters some labeled tree of depth $d$. If there is no such $d\ge 1$, that is, if there is not even a depth-1 tree, shattered by $H$ (this happens exactly when $H$ has just one function), we set $\ldim(H) = 0$. In turn, if for every $d\ge 1$ there exits a labeled tree of depth $d$, shattered by $H$, then we set $\ldim(H) = +\infty$. We leave $\ldim(\varnothing)$ undefined.

It is important to note again that by labeled trees we only mean complete trees, that is, if $H$ shatters some incomplete tree of depth $d$ (some leaves are at depth $d$ but some are at smaller depth), this does not count.

We mention two standard  properties of the Littlestone dimension that easily follow from the definition.
\begin{prop}
\label{prop_size}
    For any hypothesis class $H$ we have $\ldim(H) \le \log_2(|H|)$.
\end{prop}

\begin{prop}
\label{prop_restictions}
    Consider any hypothesis class $H$ and any $x\in X$. Define $H_0 = \{f\in H\mid f(x) = 0\}$ and $H_1 = \{f\in H\mid f(x) = 1\}$. Assume that both $H_0$  and $H_1$ are non-empty. Then
    \[\ldim(H) \ge \min\{\ldim(H_0), \ldim(H_1)\} + 1.\]
\end{prop}

Online learning of a hypothesis class $H$ is the following game, played between Learner and Adversary. First, Adversary chooses an ``objective function'' $f\in H$, without showing it to Learner. Then the game continues infinitely many rounds and in the $r$th round, $r  = 1, 2, 3,\ldots,$ the following happens: Adversary names $x_r\in X$, Learner predicts $\widehat{y_r}\in\{0,1\}$, and Adversary ``reveals'' $f(x_r)$. An online learning algorithm that makes at most $d$ mistakes on $H$ is a strategy of Learner that guarantees that the number of $r$ with $\widehat{y}_r\neq f(x_r)$ does not exceed $d$. 

\begin{remark}
    There exists another version of this game. In this version, Adversary does not choose $f\in H$ in advance, but instead, for each $r$, in the end of the $r$th round it names some $y_r\in\{0, 1\}$, under the restriction that there must exist some $f\in H$ such that $f(x_1) = y_1, \ldots, f(x_r) = y_r$. These two versions are equivalent in the sense that Learner has a strategy, guaranteeing at most $d$ mistakes, in one game if and only if it has such strategy in the other game, for any $d\in\mathbb{N}$. We point out that these two games are no longer equivalent when Learner just wants to guarantee that the number of mistakes is finite in every play (but can be arbitrarily large in different plays). Learnability in this setting also admits a combinatorial characterization through a generalization of Littlestone dimension to ordinal numbers~\cite{bousquet2021theory}.
\end{remark}

\begin{theorem}[Littlestone~\cite{littlestone1988learning}]
\label{thm:little}
    The Littlestone dimension of a hypothesis class $H$ is equal to the minimal $d$ for which there exists an online learning algorithm making at most $d$ mistakes on $H$.
\end{theorem}

\section{The model and a simple lower bound}
\label{sec_model}

In this section, we formally introduce the consistent oracle model for classes of Littlestone dimension $d$, for every $d\ge 1$. This will be modeled as a game between Learner and Adversary. Adversary is responsible for giving inputs for prediction to Learner, revealing their values, and answering queries to the consistent oracle, asked by Learner. Recall that we assume that consistency oracle may go into the infinite loop when run on a non-realizable sample. Therefore, we assume that Learner queries the consistent oracle only on samples that are surely ``realizable'', that is, on those that consist of Adversary's assignments,  used so far. Without loss of generality, the consistency oracle is run after each round, which we emulate in the game by requiring Adversary to provide a function, agreeing with all its assignments so far.

Formally, the consistent oracle model for classes of Littlestone dimension $d$ is the following perfect information game, played between Learner and Adversary. The game proceeds in infinitely many rounds, indexed by $r = 1, 2, 3, \ldots$ In the round number $r$, the following happens:
\begin{itemize}
    \item Adversary names  $x_r\in\mathbb{N}$;
    \item Learner names  $\widehat{y_r}\in\{0, 1\}$.
    \item Adversary names a bit $y_r\in\{0, 1\}$ and a function $f_r \colon X \to\{0, 1\}$ such that, first,
    \[f_r(x_1) = y_1, \ldots,  f_r(x_r) = y_r\]
    (the function $f_r$ must be consistent with all assignments of Adversary we have so far), and second,
    \[\ldim(\{f_1, \ldots, f_r\}) \le d\]
    (Adversary makes sure that its functions could come from a class of Littlestone dimension at most $d$).
\end{itemize}
\begin{remark}
Note that both players always have at least one legal move from any reachable position. Adversary, for example, can just use the same function as in the previous round.
\end{remark}
The \emph{number of mistakes} along an infinite play is the number of rounds $r$ such that $\widehat{y_r}\neq y_r$ (in some plays, of course, it might be infinite).

\emph{An online learning algorithm with consistent oracle that makes at most $n$ mistakes for classes of Littlestone dimension $d$} is a strategy of Learner in the above game, ensuring that the number of mistakes along any play is at most $n$.


We observe that there is a simple strategy of Adversary, enforcing Learner to make at least $2^{d+1} - 1$ mistakes. This shows that no online learning algorithm in the model with consistent oracle, making less than $2^{d+1} - 1$, is possible. In Section \ref{sec_lower}, we improve this lower bound to $3^d$.
\begin{prop}
For every $d\ge 1$,
    there exists no online learning algorithm that in the consistent oracle model for classes of Littlestone dimension $d$ makes at most $2^{d+1} - 2$ mistakes.
\end{prop}
\begin{proof}
    Adversary picks any $n = 2^{d+1} - 1$ elements of the domain and gives them, in any order, to Learner. Each time Learner predicts some $\widehat{y_r}$, Adversary plays the opposite label $y_r = \lnot \widehat{y_r}$, and an \emph{arbitrary function} $f_r$, agreeing with the current history. As a result, Learner makes a mistake in each round. After $n$ rounds, we have $n = 2^{d+1} - 1$ mistakes, but the set of Adversary's function still has Littlestone dimension at most $d$. This is simply because there are at most $2^{d+1} - 1$ different functions so far, so the Littlestone dimension, by Proposition \ref{prop_size}, is bounded by $\log_2(2^{d+1} - 1) < d + 1$. 
\end{proof}




\section{The algorithm}
In this section, we present our upper bound.
\label{sec_alg}
\begin{theorem}
\label{thm_main}
For every $d$,
    there exists an online learning algorithm in the consistent oracle model that for classes of Littlestone dimension $d$ makes at most $O(256^d)$ mistakes. 

Producing each prediction takes linear time in the number of mistakes made so far, assuming that it takes constant time to evaluate a function, coming from the consistent oracle, on a given $x\in X$.

\end{theorem}

It is possible to give a single algorithm that does not depend on $d$ but makes $O(256^d)$ mistakes when run in the consistent oracle model for classes of Littlestone dimension $d$. We discuss this at the end of this section, where we give a simple non-recursive implementation of this algorithm.

For the sake of clarity, we first give a proof of  \Cref{thm_main} using a recursive version of this algorithm that receives $d$ on input. It allows a straightforward inductive proof of the claimed mistake bound. 

\begin{proof}[Proof of \Cref{thm_main}]
For our algorithm, it is not necessary to run the consistent oracle every round and on all examples, revealed by Adversary so far. Instead, it will be enough to run it only after rounds with mistakes, and only for examples from these rounds. More precisely, the algorithm will maintain a sample $S = (x_1, y_1), \ldots, (x_e, y_e)$, where $e$ is the number of mistakes so far. Whenever we make a mistake, we add a new pair to the end of $S$ and we run the consistent oracle on the updated $S$. As a result, we also keep an ordered list of functions $f_1, \ldots, f_e$, given to us by the consistent oracle after rounds with mistakes. To produce a prediction on a given $x\in X$, we just need to evaluate $f_1(x), \ldots, f_e(x)$, and then the prediction takes $O(e)$-time.

The goal of Learner is to guarantee that $\ldim\{f_1, \ldots, f_e\}$ grows as $\log_{256}(e) - \Omega(1)$. Then for classes of Littlestone dimension $d$, the number of mistakes will be bounded by $O(256^d)$.

For effective implementation of our algorithm, we will mark some functions among $f_1, \ldots, f_e$ as \emph{active}. Active functions will be kept in the ordered list. The size of the list will be denoted by $L$, and $g_i$ will denote the $(i+1)$st function in the current list (the first function in the list will be $g_0$). The algorithm will employ two operations with the list of active functions: (a) add a new function to the end of the list; (b) delete some functions from the list, without changing the order of the remaining ones. Both of them are easily realizable in linear time using the list data structure.  Note that after deletion, $g_i$ might refer to a different function.

\begin{algorithm}

\caption{\voteandupdate{k}}\label{alg:vad}

receive $x\in X$ to predict\;

\bigskip

\eIf{$L \ge 2^k$}
{
predict $\widehat{y} = \mathsf{MAJORITY}(g_{L-2^k}(x), \ldots, g_{L - 1}(x))$\;
}
{
predict $\widehat{y} = 0$\;
}

\bigskip
receive $y\in\{0, 1\}$\;
\bigskip
\eIf{$y = \widehat{y}$}
{
    go to 1\;
    }
{
add $(x, y)$ to $S$\;
      \eIf{$k=0$ or $L < 2^k$}{
      $L:= L +1$\; 
      $g_L :=
      \cons(S)$\;}
      {

      among $g_{L - 2^k}, \ldots, g_{L - 1}$, select any $2^{k - 1}$, agreeing with $(x, \widehat{y})$, and delete the other $2^{k - 1}$\;
      }
      }
    \end{algorithm}

We will add a new active function only in the case when we used the last active function for prediction (that is, given $x\in X$, we predicted $g_{L - 1}(x)$) and we made a mistake. In this way, we ensure that the list of active functions never has repetitions. Indeed, consider a moment we put some function $g$  on the list. We show that we can never put the same function again. Indeed, currently, $g$ is the last one in the list. Whatever deletions we do, $g$ stays at the end of the list. Now, a new function comes when we made a prediction $\widehat{y} = g(x)$ on some $x$, and it was a mistake. At this moment, we add $(x, 1- \widehat{y})$ to $S$. This means that all future functions (including one that is about to be added) will be equal to $1 - \widehat{y}$ on $x$ as they come from the consistent oracle called on the updated $S$.

At each round of the game, the algorithm predicts according to the majority vote over the last $2^k$ active functions, for some $k\ge 0$. If there is no mistake, the algorithm keeps doing the same thing, without changing $k$ or the list of active functions. If a mistake happens, the algorithm modifies the list of active functions (and, potentially, $k$, according to the rules to be defined later). If $k = 0$ (when just the last active function is used for prediction), the algorithm runs the consistent oracle on $S$, updated after the new mistake, and adds the resulting function to the end of the list of active functions. If $k > 0$, the algorithm does not add a new active function, and this ensures, as we discuss below, that the list of active functions never has repetitions. Instead, when $k > 0$, deletes some active functions: among $2^k$ active functions it used for voting, it selects any $2^{k - 1}$ of them that agree with its prediction (and, thus, disagree with the label that Adversary gave to us after our prediction) and removes other $2^{k - 1}$ from the list of active functions. This is formally defined as a subroutine $\voteandupdate{k}$ in Algorithm \ref{alg:vad}.

    Technically, if there are not enough active functions to do the requested majority vote, our algorithm just predicts $\widehat{y} = 0$. We will make sure that this can only happen for $k = 0$. Let us also remark that we exit $\voteandupdate{k}$ exactly after 1 mistake since the start of the subroutine (and if we never make mistakes inside it, we never exit it).

    We now describe the recursive implementation of our algorithm. The idea is to construct a procedure that halts whenever $O(256^d)$ mistakes are made, and if it halts, the set of active functions is guaranteed to have the Littlestone dimension larger than $d$. Such a procedure cannot halt over a class of Littlestone dimension $d$ (all active functions come from the class) and hence it makes less than $O(256^d)$ mistakes.
    
    Ideally, to facilitate an inductive proof, the algorithm, given parameter $d$, would make a constant number of recursive calls to its instance with parameter $d-1$ and so on. However, to make the induction work, we need a stronger property than just ``having Littlestone dimension $> d$''. This property is given in the next definition.
    

\begin{definition}
    Take any real $\gamma\ge 1$.
A non-empty set $T\subseteq \{0, 1\}^X$ is \mbox{\textbf{$\gamma$-advanced}} if for every non-empty $A\subseteq T$ we have
\[\ldim(A) \ge \gamma + \log_{16}\left(\frac{|A|}{|T|}\right).\]
\end{definition}
By definition, any $\gamma$-advanced set has Littlestone dimension at least $\gamma$ (use the definition for $A = T$). Correspondingly, our plan is to write a subroutine $\createadv{k}$ (see Algorithm \ref{alg:createadv}) that, if it halts, ``creates'' a $(1+k/2)$-advanced set of active functions. For classes of Littlestone dimension $d$, it will be enough to run $\createadv{2d - 1}$.



\begin{algorithm}
\caption{\createadv{k}}\label{alg:createadv}
\eIf{$k=0$}{
      \For{$i:=1$ \KwTo $16$}{
      \voteandupdate{0}
      }
      }
    {
    \For{$i:=1$ \KwTo $16$}{
    \createadv{k-1}\;
      \voteandupdate{3k + 1}}
    }
\end{algorithm}

The following proposition summarizes the properties of $\createadv{k}$.
\begin{prop}
\label{prop_create}
    For every $k \ge 0$, procedure \createadv{k} halts exactly after $R_k = 16 + 16^2 +\ldots + 16^{k + 1}$ mistakes since the start of the procedure. Moreover, if it halts, it attaches to the list of active function exactly $2\cdot 8^{k +1}$ new active functions, forming a $(1+k/2)$-advanced set (and it does not delete active functions that were in the list before the start of the procedure).
\end{prop}

This proposition implies \Cref{thm_main}. Namely, we claim that the procedure \createadv{2d - 1} is an online learning algorithm in the consistent oracle model that for classes of Littlestone dimension $d$ makes $O(256^d)$ mistakes. Indeed, the algorithm \createadv{2d - 1} halts after precisely $16 + 16^2 + \ldots + 16^{2d}$ mistakes. On the other hand, it cannot halt when run it in the consistent oracle model for Littlestone dimension $d$. This is because it can only halt when there is a $(1 + (2d - 1)/2) = (d+1/2)$-advanced set of functions, coming from the consistent oracle, and the Littlestone dimension of a $(d+1/2)$-advanced set is larger than $d$. Thus, \createadv{2d - 1} cannot make $16 + 16^2 + \ldots +16^{2d} = O(256^d)$ mistakes in the consistent oracle model for Littlestone dimension $d$.
It remains to prove \Cref{prop_create}.
\begin{proof}[Proof of Proposition \ref{prop_create}]
    We prove this proposition by induction on $k$. We start with the mistake bound. For $k = 0$, the procedure \createadv{k} executes \voteandupdate{0} 16 times. Each \voteandupdate{0} halts exactly after 1 mistake, so if all of them halt, we had exactly $R_0 = 16$ mistakes. We now establish the inductive step. Assume that the mistake bound is proved for \createadv{k - 1}. Now, \createadv{k} runs 16 times \createadv{k - 1} and 16 times \voteandupdate{3k + 1}. Thus, overall, it halts after $16(R_{k - 1} + 1) = 16 (1 + 16 + \ldots 16^k) = R_k$ mistakes, as required.

    We now establish the second part of the proposition. We start with the induction base. We have to show that \createadv{0}, if it halts, attaches 16 new active functions, forming a 1-advanced set. Note that \voteandupdate{0} always adds a new active function without touching the previous ones. Thus, \createadv{0} halts by adding 16 new active functions. As we remarked in the beginning of the description of our algorithm, the list of active functions never has repetitions. Therefore, it remains to establish that any set of 16 different functions is 1-advanced.
    \begin{lemma}
        Any $T\subseteq\{0, 1\}^X$ of size $16$ is 1-advanced.
    \end{lemma}
    \begin{proof}
        Take any non-empty $A\subseteq T$. If $|A| = 1$, then $\ldim(A) = 0 = 1 + \log_{16}(|A|/|T|)$, as required. If $|A| \ge 2$, then $\ldim(A)\ge 1$, meaning that $\ldim(A) = 1 \ge 1 + \log_{16}(|A|/|T|)$, as required.
    \end{proof}
    We are now proving the induction step. Assuming that the statement is proved for $k - 1$, we establish it for $k$. Consider any run of \createadv{k} that halts. At the beginning, it runs \createadv{k - 1}. By the induction hypothesis, if it halts, it attaches a  $(1 + (k - 1)/2)$-advanced set $T_1\subseteq \{0, 1\}^X$ of size $2\cdot 8^k$ to the list of active functions. Then we run $\voteandupdate{3k + 1}$. At this moment, for prediction we use the majority vote over the last $2^{3k +1} = 2\cdot 8^k$ active functions, that is, over $T_1$. We exit $\voteandupdate{3k +1}$ when our majority vote made a mistake on some $x = x_1$. More specifically, the majority of functions from $T_1$ were equal to some $\widehat{y}_1\in\{0,1\}$ on $x_1$, and then we obtained the opposite label $y_1 = 1 - \widehat{y}_1$ from Adversary. We choose some subset $\Gamma_1 \subseteq T_1$ of size $(1/2) |T_1| = 8^k$ where all functions are equal to $\widehat{y}_1$ on $x_1$, and we delete the rest of functions of $T_1$ from the list of active function. At this point, the pair $(x_1, y_1) = (x_1, 1 - \widehat{y}_1)$ is added to the sample on which we call the consistent oracle, meaning that all new active functions will be equal to $1 - \widehat{y}_1$ on $x_1$, opposite to functions from $\Gamma_1$.

    We then repeat \createadv{k - 1} and \voteandupdate{3k + 1} 15 more times. Each time, a set $\Gamma_i$ with the same properties as $\Gamma_1$ is attached to the list of active functions. More precisely, \createadv{k} ends up attaching $\Gamma_1, \ldots, \Gamma_{16}$, where for every $i = 1, \ldots, 16$ we have the following:
    \begin{itemize}
        \item $\Gamma_i$ is a set of functions of size $8^k$, which is a subset of some $(1 + (k - 1)/2)$-advanced set $T_i$ of size $2 \cdot 8^k$;
        \item there exists $x_i\in X$ and $\widehat{y}_i\in\{0, 1\}$ such that, first, all functions from $\Gamma_i$ are equal to $\widehat{y}_i$ on $x_i$, and second, all functions from $\Gamma_{i+1}\cup\ldots \cup\Gamma_{16}$ are equal to $1 - \widehat{y}_i$ on $x_i$.
    \end{itemize}
    Overall, we attached $16\cdot 8^k = 2 \cdot 8^{k +1}$ new active functions. It remains to show that $T = \Gamma_1 \cup \ldots \cup \Gamma_{16}$ is $(1+ k/2)$-advanced.

    \bigskip

    Take any non-empty subset $A\subseteq T$. Our goal is to show that $\ldim(A) \ge 1 + k/2 + \log_{16}\left(\frac{|A|}{|T|}\right)$. Once again, the list of active functions never has repetitions, which means that $\Gamma_1, \ldots, \Gamma_{16}$ are disjoint. In particular, $A$ is a disjoint union of $\Gamma_1 \cap A, \ldots, \Gamma_{16}\cap A$. Denote
    \[\alpha = \frac{|A|}{|T|}, \qquad \alpha_i =  \frac{|A\cap \Gamma_i|}{|\Gamma_i|}.\]
    Note that our goal is to show that $\ldim(A) \ge 1 + k/2 + \log_{16}(\alpha)$.
    Since $A$ is non-empty, we have $\alpha > 0$.
    Each $\Gamma_i$ is 16 times smaller than $T$, which means that  
    \begin{equation}
    \label{eq_sum}
       \alpha = \frac{|\Gamma_1 \cap A| + \ldots + |\Gamma_{16} \cap A|}{|T|} =  \frac{\alpha_1 + \ldots +\alpha_{16}}{16}.
    \end{equation}
    First, consider the case when there exists $i = 1, \ldots, 16$ such that $\alpha_i\ge 8\alpha$. We will bound $\ldim(A)$ from below by just $\ldim(A\cap \Gamma_i)$, using the fact that $A\cap\Gamma_i\subseteq \Gamma_i\subseteq T_i$ is a sufficiently large subset of a $(1+(k-1)/2)$-advanced set $T_i$. Moreover, the improvement of the parameter will be achieved due to the fact that $A\cap\Gamma_i$ is 4 times larger w.r.t.~$T_i$ than $A$ w.r.t.~$T$.

    First of all, $A\cap \Gamma_i$ is non-empty because its size is the $\alpha_i$-fraction of the size of $\Gamma_i$, and $\alpha_i \ge 8\alpha > 0$. Again, the size of $A\cap \Gamma_i$ is the $\alpha_i$-fraction of  $|\Gamma_i|$ while the latter is half of $|T_i|$. Hence, the size of $A\cap \Gamma_i$ is the $\alpha_i/2$-fraction of $|T_i|$.
     Applying the definition for $T_i$, we get:
    \begin{align*}
        \ldim(A) &\ge \ldim(A\cap \Gamma_i) \\
        &\ge (1 + (k -1)/2) + \log_{16}(\alpha_i/2)\\
        &\ge  (1 + (k -1)/2) + \log_{16}(4\alpha) \\
        &= 1 + k/2 + \log_{16}(\alpha),
        \end{align*}
as required.

     Now, consider the case when $\alpha_i < 8\alpha$ for every $i = 1, \ldots, 16$. This time, we will bound $\ldim(A)$ from below using Proposition \ref{prop_restictions}, showing that as long as for some $x\in X$, both subsets $A_0 =\{f\in A\mid f(x) = 0\}$ and $A_1 = \{f\in A\mid f(x) = 1\}$ have Littlestone dimension at least $d$, the set $A$ itself has Littlestone dimension at least $d+1$. A crucial observation here is that for any $i < j$, we have that all functions from $A\cap \Gamma_i$ are equal to $\widehat{y}_i$ on $x_i$ while all functions from $A\cap \Gamma_j$ are equal to $1 - \widehat{y}_i$ on $x_i$. By Proposition \ref{prop_restictions}, for every $i < j$, as long as both $A\cap \Gamma_i, A\cap \Gamma_j$ are non-empty, we get:
     \[\ldim(A) \ge \min\{\ldim(A\cap \Gamma_i), \ldim(A\cap \Gamma_j)\} + 1.\]
Observe that we get the required bound $\ldim(A) \ge 1 + k/2 + \log_{16}(\alpha)$ as long as we have:
\[\ldim(A\cap \Gamma_i) \ge  k/2 + \log_{16}(\alpha)\]
for at least two different $i$.

Once again, $\ldim(A\cap \Gamma_i)$ can be bounded from below using the fact that $A\cap \Gamma_i$ is a subset of $T_i$ whose size is the $(\alpha_i/2)$-fraction of $|T_i|$. For every $i$, as long as $\alpha_i > 0$, we get
\begin{align*}
    \ldim(A\cap\Gamma_i) &\ge (1 + (k - 1)/2) + \log_{16}(\alpha_i/2) \\&= k/2 + \log_{16}(2\alpha_i).
\end{align*}
Thus, our job is done as long as $\alpha_i \ge \alpha/2$ for at least 2 different $i$. To show this, among  $\alpha_1, \ldots, \alpha_{16}$, take two largest numbers. Namely, let it be $\alpha_i \ge \alpha_j$ for some $i\neq j$. We claim that $\alpha_i \ge \alpha_j \ge \alpha/2$. Indeed, note that the sum $\alpha_1 + \ldots + \alpha_{16}$ can be bounded from above by $16 \alpha_j + \alpha_i$ (all numbers, except $\alpha_i$, are at most $\alpha_j$). Thus, from \eqref{eq_sum}, we get:
    \[\alpha_j + \frac{\alpha_i}{16}\ge \alpha.\]
    Remembering that $\alpha_i < 8\alpha$, we obtain $\alpha_j \ge \alpha/2$.
\end{proof}
\end{proof}
\subsection*{Non-recursive implementation which is independent of $\mathbf{d}$.}
Finally, we give a simple non-recursive implementation of our algorithm. Let us first observe that it is possible to give an implementation that does not depend on $d$. Indeed, for every $k\ge 0$, if we get rid of the recursion, the procedure \createadv{k} makes calls the VoteAndUpdate procedures in some fixed order. Moreover, \createadv{k} starts with a call to \createadv{k - 1}, which in turn starts with a call to \createadv{k - 2}, and so on. This means that there exists an infinite sequence of the VoteAndUpdate procedures such that, for every $k\ge 0$, some prefix of this sequence is a realization of \createadv{k}.

It remains to explicitly define this sequence.
For a moment, 
\begin{align*}
    \proc{0} &= \voteandupdate{0},\\
    \proc{k} &= \voteandupdate{3k +1},\,\, \text{for } k\ge1.
\end{align*}
 Note that \createadv{0} consists of 16 repetitions of \proc{0}, while \createadv{k} for $k>0$ consists of 16 repetitions of \createadv{k-1}; \proc{k}.

For $k > 0$,
we call \proc{k} always just after executing \createadv{k - 1}. Now, each call to   \createadv{k - 1} has exactly 16 calls to \proc{k - 1}, with the last one of these 16 calls happening right at the end of \createadv{k - 1}. In other words, for $k > 0$, the call to \proc{k} happens exactly at moments when the last called procedure is \proc{k-1}, and the total number of calls to \proc{k-1} is a multiple of 16.

This gives the following rule. If the last procedure so far is \proc{k} for some $k\ge 0$, then if the total number of calls to \proc{k} is a multiple of 16, we run \proc{k+1}, otherwise, we run \proc{0}.

This order is modeled by Algorithm \ref{alg:predict}.
\begin{algorithm}
\caption{\predict}\label{alg:predict}
$N$ := $1$\;
$i$ := maximal $i\ge 0$ such that $16^i$ divides $N$\;
run \proc{0} = \voteandupdate{0}\;
\For{$j:=1$ \KwTo$i$}
{run \proc{j} = \voteandupdate{3j+1}\;}
$N:=N+1$\;
go to 2\;
\end{algorithm}

In this algorithm, we run \proc{0} for all $N$, \proc{1} for all $N$ that are multiples of $16$, \proc{2} for all $N$ that are multiples of $16^2$, and so on. In this way, every call to \proc{k} is preceded by (and follows right after)  exactly 16 calls to \proc{k-1}, as required.

\section{The $3^d$ lower bound}
\label{sec_lower}
\begin{theorem}
For every $d\ge 1$,
    there exists no online learning algorithm in the consistent oracle model for classes of Littlestone dimension $d$ that makes less than $3^d$ mistakes.
\end{theorem}
\begin{proof}
\textbf{Notation}.
Without loss of generality, we assume that the domain $X$ is equal to $\mathbb{N}$. In the proof, we identify numbers from $\{0, 1, \ldots, 3^d - 1\}$ with strings from $\{0, 1, 2\}^d$ by assigning each   $x\in\{0, 1, \ldots, 3^d - 1\}$ the unique ternary string $x_{d - 1}\ldots x_1 x_0\in\{0, 1, 2\}^d$ of length $d$ such that:
\[x = x_{d - 1} \cdot 3^{d - 1} + \ldots + x_1\cdot 3 + x_0.\]
In other words, $x_{d-1}\ldots x_1 x_0$ is the ternary expansion of $x$, where $x_i$ refers to the digit before $3^i$.
\bigskip

\textbf{The strategy of Adversary.}
We give a strategy of Adversary, which uses a set of functions of the Littlestone dimension at most $d$ and forces the Learner to make at least $3^d$ mistakes. 

The Adversary will be playing numbers $0, 1, \ldots, 3^d - 1$ in the increasing order. When Learner receives $r\in\{0, 1, \ldots, 3^d - 1\}$, it predicts
    some $\widehat{y_r}\in\{0, 1\}$. Adversary always plays the opposite label $y_r = \lnot \widehat{y_r}$. Adversary also has to give some function $f_r\colon\mathbb{N}\to\{0,1\}$, satisfying:
    \begin{equation}
    \label{y_x}
        f_r(0) = y_0, \ldots, f_r(r) = y_r.
    \end{equation}
The function $f_r$ is defined as follows. We only have to define $f_r(x)$ for $x > r$. If $x \ge 3^d$, we set $f_r(x) = 0$. Now, consider the case when $0 \le r < x \le 3^d - 1$. Let $r_{d - 1}\ldots r_0$ and $x_{d - 1}\ldots x_0$ be ternary expansions of $r$ and $x$, respectively. Consider the largest $i$ such that $r_i \neq x_i$. In other words, $i$ refers to the most significant position in the ternary expansions of $r$ and $x$ where they differ. Since $r < x$, we have that $r_i < x_i$, meaning that $r_i\in\{0, 1\}$. Set $f_r(x) = r_i$.

    This strategy forces Learner to make $3^d$ mistakes. It remains to show that $\ldim(\{f_0, \ldots, f_{3^d - 1}\})\le d$. 
    For that, we give an online learning algorithm for the class $H = \{f_0, \ldots, f_{3^d - 1}\}$ that makes at most $d$ mistakes. By Theorem \ref{thm:little}, this implies the upper bound  $\ldim(\{f_0, \ldots, f_{3^d - 1}\}) \le d$.

    \bigskip

    \textbf{An online learning algorithm for $H = \{f_0, \ldots, f_{3^d - 1}\}$ with at most $d$ mistakes.}
    In the description of the algorithm, we only consider inputs $x\in\{0, 1, \ldots, 3^d - 1\}$, because on all $x\ge 3^d$, all functions from the class are equal to $0$.
    
    Assume that the true ``objective function'' used by the Adversary is $f_r\in H$ for some $r \in 0, 1, \ldots, 3^d - 1$. Of course, this $r$ is not known to the learning algorithm in the beginning. Let $r_{d - 1}\ldots r_0$ be the ternary expansion of this $r$. Before proceeding, let us introduce an auxilliary definition.
    \begin{definition}
    Let $\ell\in\{1, \ldots, d\}$.
        A number $x\in\{0, 1, \ldots, 3^d - 1\}$ is \textbf{$\ell$-informative} for $r\in\{0, 1, \ldots, 3^d - 1\}$ if $(l - 1)$ most significant digits in the ternary expansion of $x$ and $r$ coincide, and $f_r(x) \neq y_x$ (note that for $\ell = 1$, the first condition is vacuous so we are only left with the condition $f_r(x) \neq y_x$ in this case).
    \end{definition}
The algorithm maintains the following invariant for $\ell = 1, \ldots, d$: after $\ell$ mistakes, the algorithm knows some $x\in\{0, 1,\ldots, 3^d - 1\}$ which is $\ell$-informative for $r$.

In the beginning, until the first mistake, the algorithm predicts $y_x$ for every $x\in\{0, 1, \ldots, 3^d-1\}$. Any $x$ that causes the first mistake will be $1$-informative for $r$, fulfilling the invariant for $\ell = 1$. 

\bigskip
Now, observe that we can compute $r$ from any $x\in\{0, 1, \ldots, 3^d - 1\}$ which is $d$-informative for $r$. This will mean that after $d$ mistakes, our algorithm is able to identify the objective function exactly, thus making at most $d$ mistakes.

By definition, $d - 1$ most significant digits of $r$ and $x$ coincide. This means that the ternary expansion of $x$ is $r_{d -1}\ldots r_1 a$ for some $a\in\{0, 1, 2\}$. We need to figure out the least significant digit of $r$, that is, $r_0$. We already know that $f_r(x) \neq y_x$. This means that $r < x$ because by definition, $f_r(x) = y_x$ for all $x\le r$. In particular, $r_0 < a$. Hence, the case $a = 0$ is impossible, and  if $a = 1$, then $r_0 = 0$. Now, if $a = 2$, then $r_0\in\{0, 1\}$. Additionally, by definition of $f_r$, we have $f_r(x) = r_0$. Knowing that $f_r(x)\neq y_x$, we conlude that $r_0 = \lnot y_x$.

\bigskip

It remains to show how to maintain the invariant. More specifically, for every $\ell = 1, 2, \ldots, d - 1$ we show the following. Assume that the algorithm knows some $x\in\{0, 1, \ldots, 3^d - 1\}$ which is $\ell$-informative for $r$. Then it is possible to make predictions in such a way that after any mistake the algorithm will know some $\widehat{x}\in\{0, 1, \ldots, 3^d - 1\}$ which is $(\ell+1)$-informative for $r$.

If $x$ is $\ell$-informative for $r$, then $f_r(x) \neq y_x$ and the ternary expansion of $x$ starts with $r_{d - 1}\ldots r_{d - \ell +1}$ (and these digits are already known to the algorithm). Assume that the algorithm receives some $z\in\{0, 1,\ldots, 3^d - 1\}$. Let $z_{d - 1}\ldots z_0$ be its ternary expansion. We start with the following observation: if  $z_{d - 1} \ldots z_{d - \ell +1}\neq r_{d - 1}\ldots r_{d - \ell +1}$, then the algorithm can uniquely determine the value $f_r(z)$. Indeed, take the largest $i\in\{d - 1, \ldots, d - \ell + 1\}$ such that $z_i \neq r_i$. If $z_i < r_i$, then $z < r$, meaning that $f_r(z) = y_z$ (observe that $z$ and hence $y_z$ are known to the algorithm). Now, if $z_i > r_i$, then $z > r$ and $f_r(z) = r_i$, and the value of $r_i$ is already known to the algorithm. Overall, if  $z_{d - 1} \ldots z_{d - \ell +1}\neq r_{d - 1}\ldots r_{d - \ell +1}$, the algorithm can avoid making a mistake, maintaining the invariant trivially.

Assume now that $z_{d - 1} \ldots z_{d - \ell +1} =  r_{d - 1}\ldots r_{d - \ell +1}$. Let $a, b\in\{0, 1, 2\}$ be the $l$-th most significant digits of $x$ and $z$, respectively. In other words, $x$ and $z$ can be written as follows in the ternary expansion:
\[x =  r_{d - 1}\ldots r_{d - \ell +1} a\ldots, \qquad z =  r_{d - 1}\ldots r_{d - \ell +1}b\ldots\]

If $a = 0$, we observe that $x$ is already $(\ell+1)$-informative for $r$. Indeed, since $f_r(x) \neq y_x$, we have $r < x$. This can only happen if $r_{d - l} = 0 = a$, meaning that $\ell$ most significant digits of $r$ and $x$ coincide. In this case, the algorithm can predict $f_r(z)$ arbitrarily because the invariant for $\ell+1$ is already fulfilled.

If $b = 0$, the algorithm predicts that the value $f_r(z)$ is $y_z$. Assume that this leads to the mistake. We claims that $z$ is $(\ell+1)$-informative for $r$. Since $f_r(z) \neq y_z$, it remains to show that $\ell$ most significant digits of $r$ and $z$ coincide, meaning that $r_{d - l} = b = 0$. This is because $f_r(z) \neq y_z$ implies $r < z$.

Assume now that $a\le b$. Then the algorithm predicts $\lnot y_x$ for $f_r(z)$. We claim that if this leads to a mistake, then $x$ is $(\ell+1)$-informative for $r$. For that, it is enough to show that $r_{d - \ell} = a$. Assume for contradiction that $r_{d- \ell}\neq a$. Then, since $r < x$, we have $r_{d - \ell} < a\le b$. This means that $f_r(x) = f_r(z) = r_{d - \ell}$. However, $f_r(x) \neq y_x$ and $f_r(z) = y_x$, because we predicted $\lnot y_x$ for $f_r(z)$ and it was a mistake, a contradiction.  

We are only left with the case in which $a\neq 0$, $b \neq 0$ and $a > b$. This implies that $a = 2, b = 1$.

Assume first that $y_x = 0$. Then the algorithm predicts $y_z$ for $f_r(z)$. Assume that it leads to a mistake, meaning that $f_r(z) \neq y_z$. We claim that in this case, $z$ is $(\ell+1)$-informative for $r$. It remains to show that $r_{d -\ell}  = b = 1$. Assume for contradiction that $r_{d -\ell}  \neq  b$. Then, since $f_r(z) \neq y_z$, we have $r < z$ and hence $r_{d -\ell}  <  b = 1$, meaning that $r_{d- \ell} = 0$. But then, since $a = 2$, we have $f_r(x) = 0 = y_x$, a contradiction.

Finally, assume that $y_x = 1$. Then the algorithm predicts 0 for $f_r(z)$. Assume that it leads to a mistake. We claim that in this case, $x$ is $(\ell+1)$-informative for $r$. To show this, we have to show that $r_{d - \ell} = a = 2$. Assume for contradiction that $r_{d - \ell} \neq 2$. Then $f_r(x) = r_{d - \ell} = \lnot y_x = 0$. Since $b = 1$, this implies that $f_r(z)  = r_{d - \ell} = 0$. However, we predicted $0$ for $f_r(z)$ and it was a mistake, a contradiction.

\end{proof}

\end{document}